%% file: main.tex
\documentclass{article}
\usepackage{arxiv}
\usepackage[utf8]{inputenc} 
\usepackage[T1]{fontenc}    
\usepackage{amsthm}

\title{How good is PAC-Bayes at explaining generalisation?}

\usepackage{graphicx}
\usepackage{xcolor}
\usepackage{url}

\usepackage{booktabs}
\usepackage{amssymb}
\usepackage{amsmath}
\usepackage{amsfonts}
\usepackage{mathrsfs}
\usepackage{bbm}
\usepackage{natbib}
\usepackage{comment}
\usepackage{placeins}

\newtheorem{theorem}{Theorem}
\newtheorem{lemma}{Lemma}
\newtheorem{definition}{Definition}
\newtheorem{corollary}{Corollary}

\author{
  Antoine Picard-Weibel\\
  Inria and SUEZ\\
  Le Pecq, France\\
  \texttt{apicard.w@gmail.com}
    \And
  Eugenio Clerico\\
  Universitat Pompeu Fabra\\Barcelona, Spain\\
  \texttt{eugenio.clerico@gmail.com}
    \And
  Roman Moscoviz\\
  SUEZ\\Narbonne, France
    \And
  Benjamin Guedj\\
  Inria and UCL\\France and UK\\
}

\usepackage{cleveref}
\input{macros}

\input{sections}
\begin{document}

\maketitle
\begin{abstract}
We discuss necessary conditions for a PAC-Bayes bound to provide a meaningful generalisation guarantee. Our analysis reveals that the optimal generalisation guarantee depends solely on the distribution of the risk induced by the prior distribution. In particular, achieving a target generalisation level is only achievable if the prior places sufficient mass on high-performing predictors. We relate these requirements to the prevalent practice of using data-dependent priors in deep learning PAC-Bayes applications, and discuss the implications for the claim that PAC-Bayes ``explains'' generalisation.
\end{abstract}


\input{body}

\bibliographystyle{abbrvnat}
\bibliography{bibliography}
\clearpage
\appendix
\input{appendix}






\end{document}

%% file: macros.tex
\newcommand{\ADDREF}{\textcolor{red}{REF TO ADD}}

\newcommand{\R}{\mathbb{R}}

\newcommand{\norm}[1]{\left\lVert#1\right\rVert}
\newcommand{\pth}[1]{\left( #1 \right) }
\newcommand{\bk}[1]{\left[ #1 \right]}
\newcommand{\curl}[1]{\left\{#1\right\}}

\newcommand{\Id}{\textup{Id}}

\newcommand{\PB}{B}

\newcommand{\risk}{R}
\newcommand{\trisk}{\tilde{\risk}}

\newcommand{\proba}{\pi}
\newcommand{\prior}{\proba_{\textup{p}}}

\newcommand{\KL}{\text{KL}}

\newcommand{\param}{\gamma}

\newcommand{\A}{\mathcal A}
\newcommand{\dd}{\mathrm d}

\newcommand{\leqst}{\leq_{\mathrm{stoch}}}

\newcommand{\PBt}{b}
\newcommand{\Bb}{\mathcal B}
\newcommand{\F}{\mathcal F}
\newcommand{\D}{\mathcal D}
\newcommand{\PBmin}{\PB^{\min}_{D, \PBt}}

%% file: sections.tex
\newtheorem{assump}{Assumption}

\crefname{assump}{Assumptions}{assumptions}
\crefname{corollary}{Corollary}{corol.}

%% file: body.tex
\section{Introduction}
The widespread use of modern neural networks for high-stakes applications requires safety guarantees on their performance on future data, which have not been observed during the training \citep{Xu2018, russll2020}. A well-established approach to train and evaluate the performance of a predictor consists of the following steps. First, the available data are split in a train and a test datasets. The training data are used to construct the predictor, whose performance is then assessed on the test data (\emph{empirical test risk}). Finally, concentration inequalities \citep{boucheron2013concentration} are used to derive, from this finite-sample test, an upper bound on the model expected performance over the data distribution (\emph{population risk}) \citep{Langford2005a}. In the case of i.i.d.~data, this approach yields perfectly valid high-probability bounds on the performance of the learner on future data. Yet, this comes at the cost of sacrificing a fraction of the data when training. To obtain tight bounds on the population risk, this fraction cannot be too small: for $n_\textup{test}$ samples reserved for the test set, tight concentration inequalities will have deviations of order $n_\textup{test}^{-1}$ when the test error is almost $0$, and $n_\textup{test}^{-1/2}$ in the general case. In practice, 10 to 35 percent of the data is typically reserved for testing \citep{Gholamy2018}.

Various strategies have been developed to provide safety guarantees on a model performance, without the need to set aside any data at training time. These approaches typically aim to establish a \emph{generalisation bound}, namely an upper bound on the gap between empirical risk on training data and population risk. Classical complexity notions, such as the VC dimension \citep{Vapnik2000, shalev2014understanding}, consider the overall overfitting susceptibility of an entire class of predictors and allow for deriving high-probability non-asymptotic bounds that apply universally to all predictors (and thus any algorithm) in the considered class. However, for high-dimensional problems, these bounds are often loose and fail to provide meaningful guarantees for modern neural architectures. In contrast, PAC-Bayes theory has gained significant attention after yielding non-vacuous empirical bounds for over-parametrised classifiers on standard image recognition benchmarks such as MNIST, all the while using the entirety of the available data for training the model \citep{guedj2019primer, alquier2024user}. The essence of the PAC-Bayes strategy is to share the generalisation abilities of a fixed predictor to all other predictors through penalisation. This is achieved by considering a randomised predictor (called \emph{prior}, and usually required to be data-agnostic) as a reference, and then comparing it with the final predictor (a.k.a.~\emph{posterior}), learned from the dataset.

The first instance of PAC-Bayes' ability to establish non-vacuous generalisation guarantees for deep neural networks was presented by \cite{dziugaite2017a}, who relied on a data-dependent prior selected using a union bound argument. This rekindled interest in the use of data-dependent priors to tighten PAC-Bayes bounds, an idea that had been previously proposed by \cite{ambroladze2006} and \cite{parrado2012}. \cite{dziugaite2021a} considered splitting the data\footnote{We note however that both datasets in which the data are split will be used for the training, differently from what previously done when splitting into train and test datasets.} into a \emph{prior training} set and a \emph{posterior training} set (an idea already present in \citealp{parrado2012}), and reported improved PAC-Bayes bound. Tighter results were lately achieved by \cite{perez-ortiz2021a} and further improved by \cite{clerico22a}, always leveraging a data-splitting strategy for their tightest generalisation guarantees on MNIST and CIFAR datasets. These results strongly hint that a ``good'' prior is necessary if a ``good'' posterior is to be obtained. 

Despite the interest that these results have raised
, it is often overlooked that these generalisation bounds usually closely match the guarantees that one would obtain by only training the model on the prior dataset, using the posterior dataset as a testing dataset (cf.~\Cref{table:test_bounds_vs_PACBayes} in Appendix \ref{appendix:test_bounds}).
To improve generalisation guarantees, the prior constructed using the first part of the data must generalise on the second part of the data. Moreover, constructing such priors often involves a choice of variance, which, if low (as needed for tight guarantees), essentially forces the posterior to remain approximately equal to the prior. In such experiments, the PAC-Bayes approach appears an unduly complex way to get test bounds (\emph{almost no learning at the posterior level}), and not even the tightest ones.



This casts doubt on whether such PAC-Bayes methods truly deepen our understanding of the learning mechanisms, a central goal of generalisation theory. In the present work, we aim at shading light on this issue, by formalising and quantifying the burden on the prior necessary to obtain good generalisation guarantees. For a wide range of existing PAC-Bayes bounds, we show that the generalisation guarantee depends solely on the prior distribution on the empirical risk, and does not take benefit from any factor bearing insights on the generalisation power of the predictor space. For such bounds, to reach a reasonably tight generalisation level, the prior distribution must put a non-negligible amount of mass to low risk predictors. Such findings advocate systematic integration of further theoretical insights on the generalisation potential of predictors, in order for PAC-Bayes to truly explain generalisation.




\paragraph{Structure.} \Cref{sec:setting} defines the main notations we use and describes the learning framework that we consider. \Cref{sec:PACB} introduces PAC-Bayes bounds. \Cref{sec:impact} presents our main contribution, a framework to derive quantitative necessary conditions on the prior to achieve tight PAC-Bayes generalisation bounds. \Cref{sec:Cat_bound_prior_requirement} applies this framework to a bound from \cite{Catoni2004b}. \Cref{sec:PB_deep_learning_prior} discusses some implications of our results on the meaningfulness of PAC-Bayes for explaining deep learning generalisation. Finally, \Cref{sec:conc} addresses some limitation and outlines possible directions of future work. All omitted proofs from the main text are detailed in the appendix.

\section{Notation and setting}\label{sec:setting}
\paragraph{Notation.}
We always endow $\R$ (and its Borel subsets) with the standard Borel sigma-field. Let $(\mathcal A, \Sigma_{\mathcal A})$ be a measurable space, and denote as $\Pi_\mathcal{A}$ the set of all probability measures on it. For $\pi\in\Pi_\A$, we write $a\sim\pi$ when a random variable $a$ is distributed according to $\pi$. We let $\mathcal F_\A^+$ denote the space of non-negative measurable functions from $\A$ to $\R^+=[0,+\infty)$. For a function $f\in\mathcal F_\A^+$, we let $\pi[f(a)]$ (often shortened as $\pi[f]$, when no confusion arises) be the expectation of $f(a)$ under $a\sim\pi$ (which might be infinite). The conditional expectation of a random variable (or function) $X$ with respect to another random variable $Y$ is noted $\pi\bk{X\mid Y}$. For $A\in\Sigma_\A$, we denote its probability under $\pi$ as $\pi\pth{A}$. For two measures $\pi_1$, $\pi_2 \in\Pi_{\mathcal{A}}$, we say that $\pi_1\ll\pi_2$ (in words, $\pi_1$ is absolutely continuous with respect to $\pi_2$) if $\pi_1\pth{A} = 0$ for every $A\in\Sigma_\A$ such that $\pi_2\pth{A}=0$. When $\pi_1\ll\pi_2$, we denote as $\frac{\mathrm{d}\pi_1}{\mathrm{d}\pi_2}$ the Radon-Nikodym derivative (a.k.a.~density) of $\pi_1$ with respect to $\pi_2$. Given $g:\A\to\R$ (measurable) and $\pi\in\Pi_\A$, we let $\pi^{\#g}$ denote the push-forward on $\R$ of $\pi$ through $g$, namely the probability distribution on $\R$ that is the law of $g(a)$ for $a\sim\pi$. 

When $\A$ is $\R^+$, we introduce the notion of stochastic (partial) ordering on $\Pi_{[0,1]}$. Specifically, for $\pi_1$ and $\pi_2$ in $\Pi_{\R_+}$, we write $\pi_1\leq_{\mathrm{stoch}}\pi_2$ if, for all $x\in\R_+$, $\pi_1\pth{[x,\infty)}\leq \pi_2\pth{[x,\infty)}$ (equivalently, the cumulative distribution function of $\pi_1$ dominates everywhere that of $\pi_2$).

Finally, we recall that the relative entropy $\KL$ between $\pi_1,\pi_2\in\Pi_\A$ is defined as $\KL(\pi_1,\pi_2)= \pi_2\bk{\frac{\dd\pi_1}{\dd\pi_2}\log\frac{\dd\pi_1}{\dd\pi_2}}$ if $\pi_1\ll\pi_2$, and $+\infty$ otherwise. We note that the relative entropy is always non-negative, and null only if $\pi_1=\pi_2$. 

\paragraph{Setting.}
We let $(\mathcal S, \Sigma_\mathcal S)$ denote a generic measurable space,  and we fix a data-generating probability measure $\mathbb P\in\Pi_\mathcal S$. We assume that the dataset $s$ used to train the model of interest lies in $\mathcal S$ and is sampled from $\mathbb P$, namely $s\sim\mathbb P$.  

We consider a measurable space $(\Gamma, \Sigma_\Gamma)$, which we call \emph{predictor space} $\Gamma$. To evaluate how good a predictor is, we introduce a measurable\footnote{Measurability is here w.r.t.~the product sigma-field on $\Gamma\times\mathcal S$.} function $\ell:\Gamma\times\mathcal S\to\R^+$, which can be thought of as a way to account for how many and how severe mistakes the predictor makes. As we will always consider a dataset $s\sim\mathbb P$, one can see $\gamma\mapsto \ell(\gamma, s)$ as a random field indexed on the predictor space. 
For convenience, we will just write $\risk(\gamma)$ for $\ell(\gamma, s)$, making the dependence on the observed data implicit. We call such $R$ the \emph{empirical risk function}. The \emph{population risk function} $\trisk$ is defined as the point-wise average (under $s\sim\mathbb P$) of the empirical risk $\risk$, namely $\trisk(\gamma) = \mathbb P\bk{R(\gamma)} = \mathbb P\bk{\ell(\gamma, s)}$. 

For some specific instances of our results, we will focus on a more restrictive standard statistical learning framework, which we refer to as the \emph{$n$-i.i.d.~bounded setting}, where $n\geq 1$ is an integer. In this context, we assume the existence of a measurable space $(\mathcal Z, \Sigma_\mathcal Z)$ such that $\mathcal S = \mathcal Z^n$, and $\Sigma_\mathcal S=\Sigma_\mathcal Z^{\otimes n}$ is the product sigma-field. Additionally, we require the existence of a measure $\hat{\mathbb P}\in\Pi_\mathcal Z$ such that $\mathbb P = \hat{\mathbb P}^{\otimes n}$. Finally, we ask that there is a function $\hat\ell:\Gamma\times\mathcal Z\to[0,1]$ such that, for any $\gamma\in\Gamma$, $\ell(\gamma,s) = \frac{1}{n}\sum_{i=1}^n\hat\ell(\gamma, z_i)$, where $s = (z_1,\dots,z_n)$. However, let us emphasise that the majority of our theory holds without requiring this \emph{$n$-i.i.d.~bounded setting}, which will only be assumed when explicitly stated.

\section{PAC-Bayes generalisation bounds}\label{sec:PACB}
We now give an ``abstract'' definition of PAC-Bayes bound for the general learning setting that we described in \Cref{sec:setting}, and later comment on it to show how it matches with standard instantiations in the literature. For any predictor space $\Gamma$, we define the function space $\Bb_\Gamma$, as the space of all functions ${\Pi_\Gamma}^2\times\F_\Gamma^+\to\R$.
\begin{definition}\label{def:PBbound}
    Fix a predictor space $\Gamma$ and an empirical risk function $R$. $B\in\Bb_\Gamma$ is a \emph{PAC-Bayes bound} (at confidence level $\delta\in(0,1)$), if the set $\big\{\forall\pi\in\Pi_\Gamma\,,\,\pi[\tilde R]\leq B(\pi,\prior, R)\big\}$ is $\Sigma_\mathcal S$-measurable, and 
    \begin{equation}\label{eq:PAC_Bayes_bound}\mathbb P \left(\forall\pi\in\Pi_\Gamma\,,\,\pi[\tilde R]\leq B(\pi,\prior, R)\right) \geq 1-\delta\,.\end{equation}
\end{definition}
We remark that in the above definition we used the fact that $R$ is a random element of $\F_\Gamma^+$, the randomness coming from the implicit dependence on $s\sim \mathbb P$.


We note that \eqref{eq:PAC_Bayes_bound} has the typical form of a \emph{generalisation bound}, namely a high-probability upper bound on the (expected) population risk function of a randomised predictor, in terms of the (expected) empirical risk. The distribution $\prior$ is typically called \emph{prior} in the PAC-Bayes literature, and the bound holds uniformly over all the \emph{posterior} distributions $\pi\in\Pi_\Gamma$. 

Well known instances of PAC-Bayes bounds in the form \eqref{eq:PAC_Bayes_bound} have been established in the setting where the empirical risk function is an average of $n$ bounded, independent losses (what we have previously called the $n$-i.i.d.~bounded setting). A straightforward result, dating back to \cite{Catoni2004b} (see \citealp{alquier2024user} for a simple proof) and that we will henceforth refer to as \emph{Catoni's bound}, is given by
\begin{equation}
\label{eq:Catoni}
\PB_{\mathrm{Cat}, \lambda}\pth{\pi, \prior, R} = \pi\bk{\risk} + \lambda\KL(\pi, \prior) -\lambda\log(\delta)+\frac{1}{8n\lambda}\,.
\end{equation}
For any $\lambda >0$, let $\PB=\PB_{\mathrm{Cat}, \lambda}$. Then \eqref{eq:PAC_Bayes_bound} holds under the $n$-i.i.d.~bounded setting, described at the end of \Cref{sec:setting}. Fixed the prior distribution $\prior$ and the function $R\in\mathcal F_\Gamma^+$, the posterior that minimises Catoni's bound is the so called \emph{Gibbs posterior}, which we denote as $\hat\pi_\lambda$. This is defined via $$\frac{\mathrm{d}\hat{\pi}_\lambda}{\mathrm{d}\prior} \propto e^{- \lambda^{-1}\risk}\,.$$ Note that this emulates the classic Bayesian posteriors, where the log-likelihood is now replaced by the negation of the risk.


The typical pattern underlying the design of PAC-Bayes bounds was outlined by \citep{begin2016pac}: the generalisation abilities of a fixed randomised predictor (the prior)  is extended to all other randomised predictors (the posterior) through a \emph{change of measure} penalisation. We remark that an analogue strategy can be followed when dealing with losses that are Lipschitz in the predictors. Indeed, the generalisation ability of any fixed data-agnostic predictor $\param_0$ can be leveraged to obtain uniform generalisation results over all the predictors $\param$, by noticing that $\trisk(\param)\leq \risk(\param_0) - \trisk(\param_0) + \risk(\param) + 2L\norm{\param - \param_0}$. Therefore, any high probability bound on $\trisk(\param_0) - \risk(\param_0)$ allows to control the generalisation gap of all parameters $\param$ simultaneously, with a penalty proportional to the distance between $\param$ and $\param_0$. If the risk decreases faster than this penalty grows, one can construct an estimator with improved population risk guarantees. In a way, PAC-Bayes follows a similar logic, with the advantage that the (strong) Lipschitz requirement can be lifted by considering an initial guess spread out on all predictors, namely a \emph{randomised} predictor.

In the Lipschitz analogy considered, small bounds on the population risk can only be obtained if the initial guess $\param_0$ is close to a predictor with small empirical risk. More explicitly, to upper bound the population risk by $\varepsilon$ with such approach, there must exist a predictor whose empirical risk is lower than $\varepsilon$ in the ball centred on $\param_0$ of radius $\varepsilon/(2L)$. Therefore, the empirical risk of the initial guess must be small enough, that is to say less than $3\epsilon/2$. In fact, one can show that this naïve Lipschitz approach can never improve on the generalisation guarantees obtained on the initial predictor (see Appendix \ref{app:Lipschitz}).

We remark that, in contrast with the Lipschitz approach, PAC-Bayes posterior test guarantees can massively improve on the test guarantees of the prior. For an explicit example, let us consider a setting with only two predictors, $\Gamma=\{\param_0, \param_1\}$. Let $n=500$ and assume that we are in the $n$-i.i.d. bounded setting. Assume that for the observed dataset $s$ one classifier is perfect, and the other always wrong half the time ($R(\param_0) = 0$, $R(\param_1) = 1/2$). Let the prior $\prior$ be the uniform distribution that gives equal weight to both predictors. The Gibbs posterior $\hat\pi_\lambda$, with temperature $\lambda=0.01$ obtains a generalisation guarantee $\PB_{\mathrm{Cat}, \lambda}=0.062$ (with $\delta=0.05$). This improves on the average performance of the prior ($\prior[R]=1/4$), and it is thus always better than \emph{any} test bound on the prior.\footnote{Note that, as $\prior$ does not depend on $s$, the dataset $s$ can indeed be used to provide generalisation guarantees on $\prior$.}

The quality of the generalisation guarantees can moreover not directly be inferred from the average prior risk $\prior\bk{\risk}$, which is the analogue of $\risk(\gamma_0)$ in the Lipschitz case. Indeed, a Dirac prior $\widetilde{\prior}$ on predictor $\gamma_2$ such that $\risk(\gamma_2) = 0.1$ leads to a posterior distribution with generalisation guarantee of $\PB_{\textup{Cat}, \lambda} = 0.155> 0.062$, though $\prior\bk{\risk} \geq \widetilde{\prior}\bk{\risk}$.


These elementary results showcase that the information brought by the prior risk average is too crude to be really informative on the generalisation ability. Therefore, a more fine grained description of the prior must be introduced to study its impact on the posterior test guarantees. This is the object of the next section, which illustrates the main theoretical results of our work.
\section{Impact of the PAC-Bayes prior on generalisation bounds}\label{sec:impact}
This section aims to investigate and quantify necessary conditions on the prior distribution for PAC-Bayes bounds to achieve a target generalisation level. More precisely, we phrase these requirements in terms of minimal mass that the prior needs to assign to well-performing predictors to ensure that good generalisation guarantees are attainable. We proceed through a series of key insights. 

First, in \Cref{sec:PB_min_fun_of_risk_push}, we show that, for a large class of PAC-Bayes bounds, the best achievable guarantee (optimising over the posterior, for any fixed prior and risk) is a function of the push-forward of the prior by the empirical risk. As a consequence, this best generalisation guaranty does not depend on the geometry of the predictor space, but only on the distribution of the (real valued) empirical risk under the prior. 

Second, in \Cref{sec:non-dec} we show that the optimal (under posterior optimisation) PAC-Bayes generalisation guarantee is an increasing function of the prior push-forward by the risk. Simply put, we formalise the intuitive assertion that the more weight the prior puts on predictors with low risk, the stronger the generalisation guarantee.

Finally, building on these considerations, in \Cref{sec:quantiles} we investigate some necessary conditions on the prior to achieve a given target generalisation guarantee. We derive an inequality on the quantiles of the prior risk push-forward, characterising the minimum probability mass a prior must allocate to well-performing predictors. This provides a principled way to assess whether a given prior is sufficient to obtain the desired generalisation bound.

Our analysis focuses on PAC-Bayes bounds follow a broadly applicable general form, which encompasses most bounds in the literature. Specifically, we will focus on bounds that fall within the functional class defined below.
\begin{definition}\label{def:class}
    Fix a function $D:\R^+\to\R^+$ and a mapping $\PBt:\R^2\to\R$. We say that a map $B\in\Bb_\Gamma$ (where $\Gamma$ is any arbitrary predictor space) is of class $\Bb(D, \PBt)$ when, for every $\pi,\prior\in\Pi_\Gamma$ and any $R\in\F_\Gamma^+$,
    $$B(\pi,\prior, R) = \PBt\pth{\pi\bk{\risk}, \prior\bk{D\pth{\tfrac{\mathrm{d}\pi}{\mathrm{d}\prior}}}}$$ if $\pi\ll\prior$, and $B(\pi,\prior, R)=+\infty$ otherwise.
\end{definition} Note that $\Gamma$ does not appear in the definition of the class, which indeed contains function in different spaces $\Bb_\Gamma$.

Considering PAC-Bayes bounds of class $\Bb(D,\PBt)$ directly implies the following two consequences. First, the PAC-Bayes bound is only sensitive to the empirical risk function through the posterior average empirical risk $\pi[R]$. Second, if we keep the posterior average constant, the impact of the posterior is a function of a divergence. For our main results, we will require two weak conditions on the functions $D$ and $\PBt$ as clarified by the next two assumptions.



\begin{assump}
\label{assump:increasingPB}
The mapping $\PBt$ in \Cref{def:class} is non-decreasing in its first argument.
\end{assump}
\begin{assump}
\label{assump:convexD}
In \Cref{def:class}, the function $D:\R^+\to\R$ is convex, and the mapping $\PBt$ is non-decreasing in its second argument.
\end{assump}

Being of class $\Bb(D,\PBt)$ is actually a property satisfied by most PAC-Bayes bounds available in the literature (e.g., the bound \eqref{eq:Catoni}, those in \citealp{mcallester2003pac, maurer2004, begin2016pac, alquier2018simpler, ohnishi2021}). Notable exceptions include PAC-Bayes bound via Bernstein's concentration inequality, as pioneered by \cite{Tolstikhin2013} (see also \citealp{mhammedi2019a}), or results such as \cite{jang2023a}, which are not expressed in terms of empirical moments. For PAC-Bayes bounds belonging to a class $\Bb(D, \PBt)$, assumption \ref{assump:increasingPB}, which states that the PAC-Bayes bound decreases with the posterior average, is very mild and natural, and to our knowledge it is met by all bounds in the literature. Likewise, the assumption \ref{assump:convexD} seems verified by all PAC-Bayes bounds belonging to a class $\Bb(D, \PBt)$.

\subsection{PAC-Bayes minimum as a prior risk push-forward functional}
\label{sec:PB_min_fun_of_risk_push}

A key consequence of requiring that the bound is of class $\Bb(D, \PBt)$ is that the optimal posterior (the one achieving the smallest bound given a prior) must admit a Radon-Nikodym derivative with respect to the prior. Moreover, \ref{assump:convexD} and implies that such density is risk-measurable, namely it must be in the form $g\circ R$, for a measurable $g$. Indeed, for any posterior $\pi$, a \emph{data-processing} inequality that can be easily derived from \ref{assump:convexD} (see \eqref{eq:PB_conditional_expect_ineq} in Appendix \ref{app:proofthm1}) implies that $\pi_\risk^\star$, defined as the conditional expectation \begin{equation}\label{eq:pistar}\frac{\mathrm{d}\pi_R^\star}{\mathrm{d}\prior} = \prior\left[\frac{\mathrm{d}\pi}{\mathrm{d}\prior}\middle\vert R\right]\,,\end{equation} achieves a lower PAC-Bayes bound than $\pi$. Therefore, any minimiser posterior $\hat{\pi}$ of the PAC-Bayes bound must be a fixed point of $\pi\mapsto\pi_\risk^\star$. 
This results in the fact that the minimisation of the PAC-Bayes bound depends only on the \emph{push-forward} measure $\prior^{\#R}$ of the prior on the empirical risk (often we will name such push-forward distribution the \emph{risk prior}). More precisely, one can show (see proof of \Cref{thm:min_bound_fun_of_push}) that minimising the PAC-Bayes bound amounts to the minimisation, on the density function $g: \R^+\to \R^+$ (the density of $R$), of an objective $\PB_{\R^+}(g, \prior^{\# \risk})$, under the constraint $\prior^{\# \risk}\bk{g} = 1$. 

What we have discussed so far implies that optimising (over the posterior distributions on $\Gamma$) a PAC-Bayes bound of class $\Bb(D, \PBt)$ satisfying \ref{assump:convexD} amounts to searching for a posterior distribution on \emph{risk values} in $\R^+$, rather than a posterior distribution on the \emph{predictor} space $\Gamma$. This is formalised by the following theorem (proved in Appendix \ref{app:proofthm1}).

\begin{theorem}\label{thm:min_bound_fun_of_push}
Fix $D$ and $\PBt$ satisfying \ref{assump:convexD}. There is a map $\PB^{\textup{min}}_{D, \PBt}:\Pi_{\R^+}\to\R$ such that, for any $\Gamma$, any $B\in \Bb_\Gamma$ of class $\Bb(D, \PBt)$, any $R\in\mathcal F_\Gamma^+$, any $\prior\in\Pi_\Gamma$, and any $\delta\in[0,1]$, 
\begin{equation}\label{eq:pbmindef}
\inf_{\pi\ll\prior}B(\pi,\prior,\risk) = \PB^{\textup{min}}_{D, \PBt}(\prior^{\# \risk})\,.
\end{equation}
Thus, the infimum of every $B$ of class $\Bb(D, \PBt)$ is fully determined by the risk prior $\prior^{\#\risk}$.
\end{theorem} Note that $\PB^{\min}_{D, \PBt}$ is independent of $\Gamma$ and only depends on $D$ and $\PBt$. In particular, two different spaces, with different risks and prior but same push-forward yield the same infimum.

\paragraph{On the PAC-Bayes paradox for large networks.}
As we have just stressed, optimising a PAC-Bayes bound of class $\Bb(D, \PBt)$, satisfying \ref{assump:convexD}, involves searching not for a probability measure on the potentially high dimensional predictor space $\Gamma$, but for a probability measure on the reals. For a $\KL$ penalised PAC-Bayes bound, such as \eqref{eq:Catoni} or \cite{maurer2004}, this implies that the divergence term can be understood as a divergence between two probability measures on the same space of dimension $1$, irrespectively of the dimension of the predictor space. This might partially explain why the Kullback--Leibler term does not empirically increase when the dimensionality of the parameter space grows, a counter-intuitive feature that startled experimenters (see Section 5.5 in \citealp{perezortiz2021learning} and Section 7.7 in \citealp{perez-ortiz2021a})\footnote{Note that as the two articles cited below consider variational approximations of the posterior, the performance of the optimal posterior is no longer a functional of the risk prior. Still, if the optimal variational posterior obtains nearly optimal result, our analysis should hold.}. The resulting Kullback--Leibler term depends on the amount of mass, put on high risk predictors, that needs to be moved to low risk predictors before the average risk becomes small enough (namely, before the cost of moving more mass exceeds the average risk improvement). Our interpretation of the smaller $\KL$ for larger networks paradox is the following. The priors in \cite{perezortiz2021learning,perez-ortiz2021a} are learned on a prior training set. Since the training algorithm for learning the prior has good generalisation ability for all architectures, and larger networks result in lower test risk (this is empirically observed and \emph{not explained}), the risk prior puts more weight for larger networks on smaller risks on the PAC-Bayes training data. As the risk as a whole is lower, shifting the same amount of weight results in a smaller decrease of the average risk. Hence the trade-off between shifting weight and diminishing the average risk is met earlier, that is to say for a smaller $\KL$ value.

For priors which are not data-dependent, the same analysis implies that an increase in the dimension of the parameter space should result in a larger $\KL$ value \emph{if} the increase of dimensionality can be understood as an increase of confounding factors, \emph{e.g.}, if the typical prior risk is increased, though the smallest prior risk might decrease (more flexibility).

\subsection{PAC-Bayes minima as non decreasing functionals}\label{sec:non-dec}

\Cref{thm:min_bound_fun_of_push} states that the PAC-Bayes generalisation guarantee is wholly determined by the performance of a random predictor drawn from the prior (that is to say, by the prior empirical risk) and \emph{neither} by any other characteristics of the predictor space, \emph{nor} by the data generating mechanism. Let us consider two spaces of predictors, $\Gamma_1$ and $\Gamma_2$, each equipped with a prior, $\pi_1$ and $\pi_2$ respectively. If, for a fixed sample $s$, the empirical risks $\risk_1$ (defined on $\Gamma_1$) and $\risk_2$ (defined on $\Gamma_2$) are such that ${\pi_1}^{\# \risk_1} = {\pi_2}^{\# \risk_2}$, then the best PAC-Bayes guarantees will be identical. Similarly, consider a data generating mechanism outputting two datasets, $s_1$ and $s_2$. Denote $\risk_1 = \ell(\cdot, s_1)$ and $\risk_2 = \ell(\cdot, s_2)$. For any prior $\pi$ on $\Gamma$ such that $\pi^{\# \risk_1}=\pi^{\# \risk_2}$, the best PAC-Bayes guarantees are equal. Notably, this holds in the case where $s_2$ does not contain any signal (\emph{e.g.}, random label case): for the test guarantees of the posterior to be non-vacuous, the prior empirical risk must be better behaved (with high probability) in the case where there is some signal. If priors are wholly meaningless, such an assumption appears unreasonable.

The above observations raise the following questions. What minimal properties are required on the prior empirical risk for a \emph{good} test guarantee on the optimal posterior? Is there a way to gauge and compare prior empirical risks in terms of known quantities? A first remark consists in noticing that the stochastic order on measures on $\R_+$ should be preserved. That is to say, if $\forall r \in \R_+$, $\pi_1^{\# R}\pth{[0,r]}\geq \pi_2^{\# R}\pth{[0,r]}$, we expect the bound on $\pi_1^{\# R}$ to improve on the bound on $\pi_2^{\# R}$. This is indeed the case for PAC-Bayes bounds satisfying \ref{assump:increasingPB} and \ref{assump:convexD}, as it is clarified by the next theorem (see Appendix \ref{app:proofthm2} for the proof).

\begin{theorem}\label{thm:PB_min_increase_stoch_order}
    Fix $D$ and $\PBt$ satisfying \ref{assump:increasingPB} and \ref{assump:convexD}. Then, $\PB^{\min}_{D, \PBt}$ is increasing with respect to the stochastic ordering. More explicitly, for any $\rho_1$ and $\rho_2$ in $\Pi_{\R_+}$, we have the implication
    $$\rho_1 \leq_{\mathrm{stoch}} \rho_2 \implies \PB^{\textup{min}}_{D,\PBt}(\rho_1) \leq \PB^{\textup{min}}_{D,\PBt}(\rho_2)\,.$$
\end{theorem}
Stochastic dominance is a partial order (\emph{i.e.}, it is not always possible to use it to compare two measures, as it might be that for some $\rho_1$ and $\rho_2$ we have $\rho_1 \not\leq_{\mathrm{stoch}} \rho_2$ and $\rho_2 \not\leq_{\mathrm{stoch}} \rho_1$). As a consequence, at times, \Cref{thm:PB_min_increase_stoch_order} does not allow to compare the behaviour of the bound for distinct settings. However, \Cref{thm:PB_min_increase_stoch_order} does imply the following property. Consider a subset of (risk) priors $\tilde \Pi_{\R_+}\subseteq\Pi_{\R_+}$. Assume that there is $\tilde\rho\in\tilde\Pi_{\R_+}$ that minimises $\PBmin$ on $\tilde\Pi_{\R_+}$. Then, $\tilde\rho$ must be a \emph{minimal} element of $\tilde\Pi_{\R_+}$ (w.r.t.~the stochastic ordering), in the sense that there is no $\rho\in\tilde\Pi_{\R_+}$ such that $\rho\leq_{\mathrm{stoch}} \tilde\rho$. Notably, if there is  $\tilde\pi_\text{p}^{\#R}\in\tilde\Pi_{\R_+}$ that is a \emph{minimum} element under the stochastic order (namely it can be compared to and is majorised by \emph{every} $\pi\in\tilde\Pi_{\R_+}$), then $\pi_\text{p}$ minimises the PAC-Bayes bound (\emph{i.e.}, it is the risk prior yielding the tightest PAC-Bayes bound amongst risk priors in $\tilde\Pi_{\R_+}$). We note that if one can consider $\tilde\Pi_{\R_+}=\Pi_{\R_+}$, then $\PBmin$ is minimised by $\delta_0$, which is a risk prior as long as there is $\gamma_0$ such that $R(\gamma_0)=0$ (if $\pi = \delta_{\gamma_0}$ then $\pi^{\#R}=\delta_0$). However, since $\prior$ has to be picked without knowledge of $s$, it is unreasonable to believe that the prior puts all the mass on perfect predictors. In the next section, we discuss what are more reasonable subsets of risk priors to focus on, in order to determine what conditions a prior ought to satisfy to make a target PAC-Bayes guarantee attainable (at least by the best posterior).

\subsection{Quantile requirements on the prior}\label{sec:quantiles}
Fix a target generalisation guarantee. We consider the problem of establishing what are necessary conditions on the prior so that such target can be achieved by a given PAC-Bayes bound. The idea is to consider parameterised classes of the space of the risk priors $\Pi_{\R_+}$, which will let us deduce some necessary requirements that the prior must satisfy. For instance, a naive approach would be to partition $\Pi_{\R_+}$ in equivalence classes determined by the mean, namely sets $\tilde\Pi^\mu_{\R_+} = \{\rho\in\Pi_{\R_+}\,:\,\rho[\Id]=\mu\}$, for $\mu\in\R_+$. Then, the idea would be to find the best PAC-Bayes bound (under posterior optimisation) for each of these classes, end hence determine a condition on the prior based on the expected value of the risk. However, \Cref{thm:PB_min_increase_stoch_order} is of no help to search for the best PAC-Bayes bound for a fixed $\mu$.\footnote{Further analysis on Catoni's bound shows that, if the risk is bounded by $1$, the minimum is reached on Bernoulli distribution in that setting. Such analysis is of slight interest: PAC-Bayes bounds are sensitive to the amount of mass put near the lowest achievable risk, and not on the mass repartition on high risk values.} Indeed, one can easily verify that if $\rho_1 \leq_{\mathrm{stoch}} \rho_2$ and $\rho_2[\Id] = \rho_2[\Id]=\mu$, then it must be that $\rho_1=\rho_2$, and hence every element of $\tilde\Pi_{\R_+}^{\mu}$ is minimal with respect to the stochastic partial order. However, \Cref{thm:PB_min_increase_stoch_order} let us achieve interesting results if we consider classes of measures in $\Pi_{\R_+}$ defined by \emph{quantile requirements}. More precisely, for some $r\in\R_+$ and $q\in[0,1]$ we define $\tilde\Pi_{\R_+}^{r,q} = \big\{\rho\in\Pi_{\R_+}\,:\,\rho\pth{[r,\infty)}\geq q\big\}$. Then $\tilde\Pi_{\R_+}^{r,q}$ has a (unique) minimum element under $\leq_{\mathrm{stoch}}$ (which is majorised by every other measure in $\tilde\Pi_{\R_+}^{r,q}$). This is the scaled Bernoulli distribution $r\mathrm{Ber}(q)$, which gives mass $q$ to $r$ and $1-q$ to $0$. 

A key consequence of \Cref{thm:PB_min_increase_stoch_order}, when considering quantile restrictions, is the following.
\begin{corollary}
\label{corol:Bernoulli_minorize}
Let $\PB$ be of class $\Bb(D, \PBt)$, with $D$ and $\PBt$ satisfying \ref{assump:convexD} and \ref{assump:increasingPB}. 
Then,
\begin{align*}
\PBmin(\prior^{\# \risk}) \geq \max_{r\in\R_+} \PBmin\Big(r\mathrm{Ber}\big(\prior^{\# \risk}\pth{[r, \infty)}\big)\Big)\,,
\end{align*}
for any prior $\prior$, where $r\mathrm{Ber}(q)$ is the distribution putting mass $q$ on $r$ and $1-q$ on $0$.
\end{corollary}
\begin{proof}
    For any $r\in\R_+$, the distribution $\prior^{\# \risk}$ stochastically dominates $r\mathrm{Ber}\big(\prior^{\# \risk}\pth{[r, \infty)}\big)$. Applying \Cref{thm:PB_min_increase_stoch_order} concludes the proof.
\end{proof}
For the sake of simplicity, we will write $\PBmin(r,q)$ for $\PBmin(r\mathrm{Ber}(q))$. It follows from \Cref{corol:Bernoulli_minorize} that to obtain a generalisation guarantee at most $\PBmin(r,q)$ using the PAC-Bayes bound $\PB$, the prior \emph{must} put more than $q$ mass on values smaller than $r$. Therefore, for a generalisation bound holding with probability at least $1-\delta$, this property must hold with probability at least $1-\delta$ on the data generating mechanism. This implies the following protocol to evaluate the conditions required to obtain a certain generalisation level.

\begin{corollary}
\label{corol:prior_condition_protocol}
Consider a PAC-Bayes bound $\PB \in \Bb(D, \PBt)$, satisfying \ref{assump:convexD} and \ref{assump:increasingPB}. Fix a target generalisation guarantee $G>0$. Define 
$
Q(r,G) = \inf\curl{q\in[0,1]: \PBmin(r,q)> G}
$,
with the convention $\inf\varnothing=1$. Let $\overline{Q}(r, G)=1- Q(r,G)$. Then, any prior $\prior$ satisfying $\PBmin(\prior^{\# \risk}) < G$ must satisfy, for all $r\in\R^+$,
$$
\prior^{\# \risk}\pth{[0, r)} \geq \overline{Q}(r,G)\,.
$$
\end{corollary}
\begin{proof}
Suppose that $\prior^{\# \risk}$ satisfies $\PBmin(\prior^{\# \risk})< G$. By contradiction, assume that there exists $r\geq 0$ such that $\prior^{\# \risk}\pth{[0, r)} < \overline{Q}(r,G)$. This implies that $Q(r,G) < 1$, hence the set in the definition of $Q$ is non-empty.
As $\prior^{\# \risk}\pth{[r,\infty)} > Q(r,G)$, we have $\prior^{\# \risk} \geq_{\mathrm{stoch}} r\mathrm{Ber}(q)$ for some $q \in\big( Q(r,G),1\big)$. This implies that $\PBmin(\prior^{\#\risk}) \geq \PBmin(r, q)$ by \Cref{thm:PB_min_increase_stoch_order}. Since, for all $q_1 \geq q_2$, $r\mathrm{Ber}(q_1) \geq_{\mathrm{stoch}} r\mathrm{Ber}(q_2)$, for all $q> Q(r,G)$, we have $\PBmin(r,q) \geq G$. Hence, $
\PBmin(\prior^{\#\risk})\geq G$,
which contradicts the assumption that $\PBmin(\prior^{\#\risk})< G$.
\end{proof}
\Cref{corol:prior_condition_protocol} states that the conditions on the prior empirical risk quantiles necessary to achieve a target generalisation bound are summarised by the function $Q$, which is fully determined by $\PBmin(r,q)$. This fact implies the following PAC-Bayes bound prior requirements investigation protocol: first find the exact expression for $(r,q)\mapsto\PBmin(r,q)$, then invert the formula to obtain the quantile requirements $\overline{Q}$. In the next section, we apply this protocol to analyse Catoni's bound.

\section{Catoni's bound prior requirements}
\label{sec:Cat_bound_prior_requirement}
In this section we apply the findings that we have established thus far to the analysis of Catoni's bound at level $\delta$ for $n$ independent observations ($n$-i.i.d. bounded setting). With the notation introduced in \Cref{sec:quantiles}, recalling that the Gibbs posteriors  minimise Catoni's bound, we see that 
\begin{equation}
\label{eq:Catoni_Bernoulli_rescaled_min}
\PB_{\textup{Cat}, \lambda}^{\textup{min}}(r,q) = -\lambda\log\pth{(1-q) + q \exp(-r\lambda^{-1})} + \frac{1}{8\lambda n} - \lambda \log\delta\,.
\end{equation}
This implies that the quantile requirement function has the closed form expression
\begin{equation}
\label{eq:Qcat}
Q_{\textup{Cat},\lambda}(r,G)) 
=\min\pth{1,\max\pth{0, \frac{1 - \exp\pth{-\lambda^{-1}G + \frac{\lambda^{-2}}{8 n} - \log(\delta)}}{1 - \exp\pth{-\lambda^{-1}r}}}}\,.
\end{equation}
A few useful observations can be drawn from \eqref{eq:Qcat}. First, the quantile requirement function $\overline{Q}_{\textup{Cat}, \lambda}$ is very sensitive to the choice of $\lambda$. Notably, there is a window of temperature, between $\lambda_{\textup{min}} = (4n (G+\sqrt{F^2 + (\log\delta)/(2n)}))^{-1}$ and $\lambda_\textup{max}=(4n (G+\sqrt{F^2 - (\log\delta)/(2n)}))^{-1}$,
outside of which the quantile requirement $\overline{Q}_{\textup{Cat},\lambda}$ saturates to $1$ for all values of $r$ (\textit{i.e.}, the prior risk distribution \emph{must} be a Dirac on $0$). If $G\leq \sqrt{-\log(\delta)/(2n)}$, this range is no longer defined. In this case, the generalisation guarantee $G$ is unreachable using Catoni's PAC-Bayes bound, for any temperature.\footnote{This rate of generalisation guarantee is \emph{not} optimal, since for perfect priors a rate of $-\log(\delta)/n$ is expected. Hence, tighter PAC-Bayes bounds should lead to smaller quantile requirements.}

Moreover, as a function of $r$, $\overline{Q}_{\textup{Cat},\lambda}$ is an increasing function that quickly saturates (as $\lambda^{-1}r$ gets much larger than $1$) to the maximal quantile requirement at $\lambda$, namely
$$\overline{Q}_{\textup{max}}(\lambda, G) =\min\pth{1, \exp\pth{-\lambda^{-1}G + \frac{\lambda^{-2}}{8 n} - \log(\delta)}}
\geq \overline{Q}_{\textup{Cat},\lambda}(r, G).$$

To obtain a ``temperature free'' worst case analysis, our quantile requirement can be minimised over the temperature parameter $\lambda$. We let $\overline{Q}_{\textup{Cat}} = \inf_{\lambda} \overline{Q}_{\textup{Cat},\lambda}$. $\overline{Q}_{\textup{Cat}}$ has no closed form expression. On the other hand, the maximal quantile requirement can easily be minimised with respect to $\lambda$, resulting in an asymptotically optimal temperature $\lambda_{\textup{opt}} = 1 / (4 G\times n)$ and an asymptotic smallest requirement of $$\overline{Q}_{\textup{max}}(G) = \min(1, \exp(- 2 G^2 n - \log(\delta))).$$
While this requirement is only valid asymptotically (\emph{i.e.}, for $r\rightarrow \infty$), the fast saturation of the quantile requirements, as $r$ grows, suggests to search for a fixed, small value of $r$, such that $\overline{Q}_{\textup{Cat}}(r, G) \simeq \overline{Q}_{\textup{max}}(G)$. For any $r$ such that there exists $\lambda$ in the range $[\lambda_{\textup{min}},\lambda_{\textup{max}}]$ satisfying $r > G - \frac{\lambda^{-1}}{8n} + \lambda \log(\delta)$, the minimum of the bound on $\lambda$ is $0$. Conversely, the maximum of this lower bound on $r$ is attained for $\lambda_{\textup{thresh}}^{-1} = \sqrt{-8n \log(\delta)}$ and  $r_{\textup{thresh}} = G - 2\sqrt{-\log(\delta)/(8n)}$. Hence, $r$ values yielding non trivial quantile requirements for all temperatures should be higher than $r_{\textup{thresh}}$. The numerical evaluation displayed in \Cref{figure:Q_opt_vs_Q_asympt_opt} shows that $\overline{Q}_{\textup{Cat}}$ undergoes a sort of phase transition at $r_{\textup{thresh}}$, with $\overline{Q}_{\textup{Cat}}(r_{\textup{thresh}})$ close to $\overline{Q}_{\textup{max}}$ for all $r> r_{\textup{thresh}}$. In other words, our numerical evaluation suggests that to obtain generalisation certificate $G$, the prior distribution must put (almost) $\overline{Q}_{\textup{max}}$ weight on predictors with risk lower than (almost) $r_{\textup{thresh}}$. We formally establish a weaker version of this statement in the theorem below. We refer to Appendix \ref{app:proofthm7} for the proof.

\begin{theorem}
\label{thm:catoni_cond}
In the $n$-i.i.d.~bounded setting, for any $\lambda>0$, to achieve a generalisation certificate of $G$ at level $1- \delta$ using Catoni's bound \eqref{eq:Catoni},  the prior must put at least $q_\alpha$ mass on predictors with empirical risk smaller or equal than $r_\alpha$, where $
q_\alpha = \alpha \exp(-2G^2n - \log(\delta))$ and $r_\alpha = \frac{G}{1-\alpha}$,
for all $0<\alpha<1$ if $G\geq \sqrt{\frac{-\log(\delta)}{2n}}$, and $r_\alpha=0$, $q_\alpha=1$ otherwise.
\end{theorem}
Although our numerical evaluation suggests that there is still room for tightening the values of $q_\alpha$ and $r_\alpha$ provided by \Cref{thm:catoni_cond}, the current statement is already sufficient to effortlessly obtain \emph{exigent} conditions on the prior. For instance, to obtain a generalisation guarantee of 1.5\% error on datasets with 60000 data points with a confidence level of $0.035$ using Catoni's bound,\footnote{These values match what one would expect to be a reasonable generalisation bound on MNIST, see \emph{e.g.} \cref{table:test_bounds_vs_PACBayes}.} \Cref{thm:catoni_cond} implies that the prior must put a mass higher than $5.3\times 10^{-12}$ on predictors making less than $1.65\%$ error on the training set. We remark that the numerical analysis done in \Cref{figure:Q_opt_vs_Q_asympt_opt} shows that the stricter condition of $5.3\times10^{-11}$ mass on predictors making less than $1\%$ error ought to hold. These values of the order of $10^{-11}$ or $10^{-12}$ might seem extremely small at a first glance. However, what we have just stated is actually a far from negligible condition. For MNIST, a fully random model (each class is randomly predicted, predictions for each sample are independent) would have a chance smaller than $10^{-50000}$ of misclassifying less than $1.65\%$ of samples. For the extremely optimistic scenario where one has only access to perfect predictors up to a class permutation, and the prior chooses a permutation uniformly at random, the mass put at predictors achieving less than $1.65\%$ error for MNIST would be $2.8\times 10^{-7}$. Considering a more realistic case where the prior distinguishes $20$ clusters and label them at random (where we are still very optimistic and assume that each cluster correspond to a single exact label), this mass would decrease to $1.9\times 10^{-20}$, well below our threshold (see Appendix \ref{app:strength_quant_req} for a more extensive discussion on the magnitude of the prior mass on low risk predictors).

\section{Implications for PAC-Bayes in deep learning}

\label{sec:PB_deep_learning_prior}

Our analysis formalises and quantifies the intuition that to obtain strong PAC-Bayes bounds requires the prior to place sufficient mass on high-performing predictors. In particular, \Cref{thm:min_bound_fun_of_push} implies that good generalisation is only possible with a prior able to distinguish between purely noisy and informative datasets. As PAC-Bayes bounds hold for any data generating process, expecting a better PAC-Bayes bound in the presence of input signal implies belief that the prior risk favour good predictors. This is at odds with the notion of the prior as merely a numerical intermediary.

The requirement for the prior to allocate non-negligible mass to well-performing predictors, despite being data-agnostic, casts doubts on the use of PAC-Bayes for deep learning, where no such \emph{natural} prior exists. To clarify, consider large deep networks. In the infinite-width limit, the output of a fully connected architecture, with suitably scaled centred Gaussian distributions on its weights, is a Gaussian process (labelled on the input space), whose covariance kernel becomes trivial as the number of layers diverges \citep{SchoenholzGGS17, lee2018deep, hayou19a}. Under such prior the output of deep architectures is insensitive to any input signal, which inevitably yields loose or vacuous PAC-Bayes bounds.

Applying PAC-Bayes to deep learning, a common approach to overcome the lack of a natural prior is to use a data-dependent prior (splitting the dataset in prior and posterior training sets, $s_1$ and $s_2$), typically a Gaussian centred on the empirical risk minimiser on $s_1$. While this leads to valid, non vacuous generalisation bounds, it raises a few concerns. First, for a strong bound to be achieved, the risk minimiser on $s_1$ should already have good generalisation. So, PAC-Bayes witnesses generalisation rather than providing insight. Second, setting small prior variance to achieve smaller average for the prior risk enforces the posterior to essentially almost exactly match the prior. Finally, our analysis in Appendix \ref{appendix:test_bounds} of the empirical results from \cite{perez-ortiz2021a, perezortiz2021learning} shows that, in all but one of the fourteen reported cases, the generalisation guarantee obtained from the test bound outperformed the PAC-Bayes certificate (see \Cref{table:test_bounds_vs_PACBayes}). We believe that these observations are enough to call into question the practical advantage of PAC-Bayes for deep learning.

Building on techniques from \cite{zhou2018a}, \cite{lofti2022} obtained tight, non vacuous PAC-Bayes generalisation guarantees for deep neural networks using a data independent prior. This prior was designed to favour networks which can be compressed - minimizing the weight given to predictors which are selected during overfitting. Their PAC-Bayes generalisation bound, which follows our assumptions, indirectly benefits from such insights on the 'generalisation potential' of the predictors; the tight bounds imply that the prior puts sufficient mass on high performing predictors when there is signal. Still, these experimental evidence showcases more the potential of the compression heuristic for the study of generalisation than any understanding of generalisation at the PAC-Bayes level. We take the view that, similarly to Bernstein's PAC-Bayes bounds \citep{Tolstikhin2013} or the coin-betting approach of \cite{jang2023a}, heuristics on generalisation power should be directly incorporated into the learning objective, in order to improve the generalisation guarantees, be applicable to cases where natural priors are available, and improve the theoretical understanding of generalisation.

\section{Conclusion}\label{sec:conc}
Our work challenges whether PAC-Bayes genuinely enhance our understanding of generalisation for complex models, such as deep networks. In particular, we highlight how the celebrated non-vacuous PAC-Bayes bounds for over-parameterised models often fail to explain \emph{why} good generalisation occurs. We thus argue that integrating additional theoretical principles is paramount to ensure that PAC-Bayes not only yields valid bounds but also provides genuine insight into generalisation.

It is worth noticing analysis in \Cref{sec:Cat_bound_prior_requirement} focuses on the tractable Catoni's bound \eqref{eq:Catoni}, while PAC-Bayes applications in deep learning often rely on tighter bounds \citep{perezortiz2021learning, perez-ortiz2021a, clerico22a}. Analysing a tighter bound would lead to yet to be determined, looser quantile requirements. However, the PAC-Bayes bounds used in the publications above are still fundamentally limited by being only sensitive to the prior risk.

Finally, we remark that while we used \Cref{corol:prior_condition_protocol} to obtain \emph{necessary} condition on the prior risk to reach a given generalisation target, the same framework could also be applied to obtain \emph{sufficient} conditions. We leave this as an open direction for further analysis. 

\newpage

%% file: appendix.tex
\section{Proofs}
\subsection{On Lipschitz induced generalisation}
\label{app:Lipschitz}
Let us assume that $\risk$ is a $L$ Lipschitz function. Then it follows that $\trisk - \risk$ is a $2L$ Lipschitz function. As such, a test bound $\trisk(\gamma_0) - \risk(\gamma_0) \leq B(\delta)$ evaluated on data independent predictor $\gamma_0$ and holding with probability $1- \delta$ can be extended to all predictors on the same high probability event through $\trisk(\gamma) - \risk(\gamma) \leq B(\delta) + 2L d\pth{\gamma, \gamma_0}$, leading to $\trisk(\gamma) \leq \risk(\gamma) + B(\delta) + 2 L d\pth{\gamma, \gamma_0}$. Let us denote $O(\gamma)$ the right hand side.

The Lipschitz hypothesis also implies that $\risk(\gamma_0) \leq \risk(\gamma) + L d\pth{\gamma, \gamma_0}$. This implies that $O(\gamma) \geq \risk(\gamma_0) + B(\delta) + L d\pth{\gamma, \gamma_0}$. In $\risk(\gamma_0) + B(\delta)$, one recognizes the generalisation guarantee on $\gamma_0$ obtained through a test bound. As such, the Lipschitz transfer of generalisation guarantee can \emph{never} yield a tighter generalisation guarantee than the one constructed on the prior. 

The same negative result also holds for the Wasserstein PAC-Bayes bounds of \cite{viallard2023wasserstein} (Theorem 1 therein): training the bound on the posterior $\pi$ will result in a bound that is worse than the one obtained for the mixture of partially data dependent prior. To train the posterior in practice, the authors lessen the impact of the Wasserstein distance by adding a multiplicative factor (which is tantamount to penalized regression in the Lipschitz case). Although the bounds are of limited practical interest, the analysis can therefore still serve as a way to motivate useful objectives.

\subsection[Proof min bound fun of push]{Proof of \protect\Cref{thm:min_bound_fun_of_push}}\label{app:proofthm1}
We start by establishing a technical lemma. We recall that in \eqref{eq:pistar} we defined the probability measure $\pi_R^\star$ on $\Gamma$ via $\frac{\dd\pi_R^\star}{\prior} = \pi_p\bk{\frac{\dd\pi}{\dd\prior}\middle| R}$.
\begin{lemma}\label{lemma:gpi}
    Fix $\Gamma$. Fix a prior $\prior$ on $\Gamma$ and a posterior $\pi\ll\prior$, and a measurable risk function $R:\Gamma\to\R^+$. Then, the Radon-Nikodym derivative $g_\pi = \frac{\dd\pi^{\# R}}{\dd\prior^{\# R}}$ is well defined, and $$\frac{\dd\pi^\star_R}{\dd\prior} = g_\pi(R)\,.$$
    Moreover, for every $g\in\F^+_{\R^+}$ such that $\prior^{\#R}[g]=1$, the measure $\pi_g$ on $\Gamma$, absolutely continuous with respect to $\prior$ and defined via \begin{equation}\label{eq:defpig}\frac{\dd\pi_g}{\dd\prior} = g\circ R\,,\end{equation} is a probability measure on $\Gamma$, is a fixed point of the map $\pi\mapsto\pi_R^\star$, and satisfies $\frac{\dd\pi_g^{\#R}}{\dd\prior^{\#R}} = g$. In particular, $g_{\pi_g}=g$.
\end{lemma}
\begin{proof}
    First, note that $\pi\ll\prior$ implies that $\pi^{\# R}\ll\prior^{\#R}$, and so $g_\pi$ is well defined. Now, let $Z$ be any $R$-measurable non-negative random variable on $\Gamma$. Then, we can write $Z=\phi(R)$, for some measurable $\phi:\R^+\to\R^+$. We have
    $$\prior\bk{Zg_\pi(R)} = \prior^{\# R}\bk{\phi g_\pi}=\pi^{\#R}\bk{\phi} = \pi\bk{Z}=\prior\bk{Z\tfrac{\dd\pi}{\dd\prior}}\,.$$
    Hence, $g_\pi(R) = \prior\bk{\tfrac{\dd\pi}{\dd\prior}\middle| R} = \tfrac{\dd\pi^\star_R}{\dd\prior}$, by definition of $\pi_R^\star$.

    For the second part of the statement, fix $g$ such that $\prior^{\#R}[g] = 1$ and let $\pi_g$ be as in the statement. Then, $\pi_g$ is a probability measure because $\pi_g[1] = \prior[g\circ R] = \prior^{\#R}[g]=1$. Moreover, since its density with respect to $\prior$ is $R$-measurable, it is clear that it is a fixed point of $\pi\mapsto\pi_R^\star$. Finally, for any $Z\in\F^+_{\R^+}$ we have
    \begin{align*}
        \prior^{\#R}\bk{Z\tfrac{\dd\pi_g^{\#R}}{\dd\prior^{\#R}}} &= \pi_g^{\#R}[Z] = \pi_g[Z\circ R] = \prior\bk{(Z\circ R)\tfrac{\dd\pi_g}{\dd\prior}}\\&=\prior\bk{(Z\circ R)\times(g\circ R)} = \prior\bk{(Z\times g)\circ R} = \prior^{\#R}[Zg]\,,
    \end{align*}
    which shows that $\frac{\dd\pi_g^{\#R}}{\dd\prior^{\#R}} = g$.
\end{proof}

We can now proceed with the proof of \Cref{thm:min_bound_fun_of_push}. For   $\rho\in\Pi_{\R^+}$, let $\D_\rho$ denote the set of (density) functions $g\in\F_{\R^+}^+$, such that $\rho[g]=1$. For $\rho\in\Pi_{\R^+}$, and $g\in \D_\rho$, we define 
$$\PB_{\R^+}(g, \rho) = \PBt\pth{\rho[\Id\times g], \rho[D(g)]}$$ and let 
$$\PBmin(\rho) = \inf_{g\in\D_\rho}\PB_{\R^+}(g, \rho)\,.$$

Now, for any $r\geq 0$, from \ref{assump:convexD} we can derive the ``data-processing'' inequality \begin{equation}\label{eq:PB_conditional_expect_ineq}
\widetilde{\PB}\pth{r, \prior\bk{D\pth{\tfrac{\mathrm{d}\pi_R^\star}{\mathrm{d}\prior}}}}\leq\widetilde{\PB}\pth{r,\prior\bk{D\pth{\tfrac{\mathrm{d}\pi}{\mathrm{d}\prior}}}}\,.
\end{equation}
Indeed, the convexity of $D$ implies (Jensen's inequality) that $\prior\bk{D\pth{\tfrac{\mathrm{d}\pi_R^\star}{\mathrm{d}\prior}}} \leq \prior\bk{D\pth{\tfrac{\mathrm{d}\pi}{\mathrm{d}\prior}}}$, and \eqref{eq:PB_conditional_expect_ineq} follows from the fact that  $\PBt$ is non-decreasing in its second argument.
In particular, since $B$ is of class $\Bb(D,\PBt)$, we have
\begin{align}
\begin{split}\label{eq:PB}
    B(\pi,\prior,\risk)&\geq B(\pi^\star_{\risk},\prior,\risk)
    = \widetilde{\PB}\pth{\pi^\star_{\risk}\bk{\risk}, \prior\bk{D\pth{\tfrac{\mathrm{d}\pi^\star_{\risk}}{\mathrm{d}\prior}}} }\\
    & = \widetilde{\PB}\pth{\prior\bk{ \risk \,\tfrac{\mathrm{d}\pi^\star_{\risk}}{\mathrm{d}\prior}},
    \prior\bk{D\pth{\tfrac{\mathrm{d}\pi^\star_{\risk}}{\mathrm{d}\prior}}})}\\
    & = \widetilde{\PB}\pth{
    \prior\bk{ \risk\,g_{\pi}( \risk)},
    \prior\bk{D( g_{\pi} ( \risk))})}\\
& = \widetilde{\PB}\pth{
\prior^{\# \risk}\bk{ \Id \times g_{\pi}},
 \prior^{\# \risk}\bk{D(g_{\pi})}} = \PB_{\R^+}(g_\pi , \prior^{\# \risk})\,,
\end{split}
\end{align}
where we used \Cref{lemma:gpi} for the second equality, and we let $\Id\times g_\pi:r\mapsto r g_\pi(r)$. Notice that the RHS above only depends on $\prior^{\# R}$ and on the density function $g_\pi$.
\eqref{eq:PB} tells us that 
\begin{equation}
\label{eq:PB_min}
\inf_{\pi\ll\prior}B(\pi,\prior, R)\geq\inf_{\pi\ll\prior}\PB_{\R^+}(g_\pi, \prior^{\# R})\geq \inf_{g\in\D_{\prior^{\# R}}}B_{\R^+}(g,\prior^{\# R}) = \PBmin(\prior^{\# R})\,.
\end{equation}

To conclude the proof of \Cref{thm:min_bound_fun_of_push} we shall establish the reverse inequality. For any $g\in\D_{\prior^{\#R}}$, let $\pi_g$ be the probability measure on $\Gamma$ defined in \eqref{eq:defpig}. We have
$$\PBmin(\prior^{\# R}) = \inf_{g\in\D_{\prior^{\# R}}}B_{\R^+}(g_{\pi_g},\prior^{\# R}) = \inf_{g\in\D_{\prior^{\# R}}}B(\pi_g,\prior, R) \geq \inf_{\pi\ll\prior}B(\pi,\prior, R)\,,$$
where for the first equality follows from the fact that $g=g_{\pi_g}$ by \Cref{lemma:gpi}, while for the second equality we used that $\pi_g$ is a fixed point of $\pi\mapsto\pi^\star_R$ (again by \Cref{lemma:gpi}), and so \eqref{eq:PB_min} holds with equality.

\subsection[Proof of PB mininima increasing function of stoch order]{Proof of \protect \Cref{thm:PB_min_increase_stoch_order}}\label{app:proofthm2}

If $\rho_1\leqst\rho_2$ then there is a coupling $\hat\rho$ such that if $(X_1,X_2)\sim\hat\rho$ then $X_1\sim\rho_1$, $X_2\sim\rho_2$, and $\hat\rho\pth{X_1\leq X_2}=1$. Now, fix $g_2\in\D_{\rho_2}$ (where we use the notation introduced in \Cref{app:proofthm1}: $g_2\in\F_{\R^+}^+$ is a density under $\rho_2$, namely  $\rho_2[g_2]=1$), and let $g_1\in\F_{\R^+}^+$ such that ($\rho_1$-almost surely) $$ g_1(X_1) = \hat\rho[g_2(X_2)|X_1]\,.$$ 
Then, we have $\rho_1[g_1] = \hat\rho[g_1(X_1)] = \hat\rho[\hat\rho[g_2(X_2)|X_1]] = \hat\rho[g_2(X_2)] =\rho_2[g_2]=1$. So $g_1$ is a density under $\rho_1$, namely $g_1\in D_{\rho_1}$. We also have that
$$\rho_1[\Id\times g_1] = \hat\rho[X_1\hat\rho[g_2(X_2)|X_1]] = \hat\rho[X_1g_2(X_2)]\leq \hat\rho[X_2g_2(X_2)]=\rho_2[\Id\times g_2]\,,$$ as $X_1\leq X_2$ $\hat\rho$-almost surely, and $g_2$ is non-negative. Moreover, since $D$ is convex, by Jensen's inequality for the conditional expectation,
$$\rho_1[D(g_1)] = \hat\rho[D(\hat\rho[g_2(X_2)|X_1])]\leq \hat\rho[\hat\rho[D(g_2(X_2))|X_1]] = \rho_2[D(g_2)]\,.$$ 
In particular, since $\PBt$ is non-decreasing in its arguments, for any $g_2\in\D_{\rho_2}$ there is a $g_1\in\D_{\rho_1}$ such that $$\PBt\pth{\rho_1[\Id\times g_1],\rho_1[D(g_1)]} \leq \PBt\pth{\rho_2[\Id\times g_2],\rho_2[D(g_2)]}\,.$$ Since $\PBmin(\rho) = \inf_{g\in\D_{\rho}}\PBt\pth{\rho[\Id\times g],\rho[D(g)]}$, we conclude that $\PBmin(\rho_1) \leq \PBmin(\rho_2)$.

\subsection[Catoni cond]{Proof of \protect\Cref{thm:catoni_cond}}\label{app:proofthm7}
Consider a value $r > r_{\textup{thresh}}$. For such a $r$, the function $$\widetilde{Q}(\beta) = \frac{ \exp\pth{- \beta G + \frac{\beta^2}{8 n} - \log(\delta) }- \exp(-\beta r)}{1 - \exp(-\beta r)}$$ is always positive. Moreover, for $\beta\in [\lambda_{\textup{max}}^{-1}, \lambda_{\textup{min}}^{-1}]$, the function $\widetilde{Q}(\beta)$ remains lower or equal to  $1$. Hence for $r > r_{\textup{thresh}}$, 

\begin{equation*}
\overline{Q}_\textup{Cat}(r,G) = \inf_{\beta\in[\lambda_{\textup{max}}^{-1}, \lambda_{\textup{min}}^{-1}]} \widetilde{Q}(\beta)
\end{equation*}

Moreover, $\widetilde{Q}$ is differentiable, and values $1$ at the extremity. Hence the minima is attained on $]\lambda_{\textup{max}}^{-1}, \lambda_{\textup{min}}^{-1}[$ for some $\beta$ such that $\widetilde{Q}'(\beta) = 0$.

For clarity's sake, let us introduce $X:\beta\mapsto \exp\pth{-\beta G + \frac{\beta^{2}}{8 n} - \log(\delta)}$ and $Y:\beta\mapsto \exp\pth{- \beta r}$. Then $\widetilde{Q} = \frac{X-Y}{1-Y}$. Hence 
\begin{align*}
\widetilde{Q}' = \frac{(1-Y)(X' - Y') + (X-Y) Y'}{ (1-Y)^2} = \frac{(1-Y)X' + (X- 1) Y'}{ (1-Y)^2}.
\end{align*}
Therefore,
\begin{align*}
&0 = \widetilde{Q}'\\
\implies& 0 = (1-Y) X' + (X-1) Y'\\
\implies & X' = \pth{(1-X) \frac{Y'}{Y} + X'}Y\\
\implies & Y = \frac{X'}{(1-X) \frac{Y'}{Y} + X'}
\end{align*}
As such, at the minima,
\begin{align*}
1 - Y &= \frac{(1-X) \frac{Y'}{Y}}{(1-X) \frac{Y'}{Y} + X'}\\
X-Y &= X\pth{ \frac{(1-X) \pth{\frac{Y'}{Y}-\frac{X'}{X}} }{(1-X) \frac{Y'}{Y} + X'}}
\end{align*}
and therefore the minima of $\widetilde{Q}$ is of form
\begin{align*}
\widetilde{Q} &= X\frac{\frac{Y'}{Y}-\frac{X'}{X}}{\frac{Y'}{Y}}\\
&= X\pth{1 - \frac{G -\frac{\beta}{4n}}{r}}
\end{align*}
Now, the minima of $X$ on $\beta$ is $ \exp(-2G^2n - \log(\delta))$ (which values $\overline{Q}_\textup{max}(G,\delta, n)$ if $G\geq \sqrt{\frac{-\log(\delta)}{2n}}$). Hence the minima of $\widetilde{Q}$ is lower bounded by

\begin{align*}
\widetilde{Q} \geq \exp(-2G^2n - \log(\delta))\pth{1 - \frac{G}{r}}
\end{align*}

Choosing $r_\alpha$ such that $ 1- \frac{G}{r} = \alpha$ obtains the quantile requirement for $G\geq \sqrt{\frac{-\log (\delta)}{2n}}$. For $G<\sqrt{\frac{-\log(\delta)}{2n}}$, the generalisation level can never be attained. Hence the set of priors reaching the generalisation level is empty. Any statement holds on the empty set.

\clearpage

\section{Test bounds versus PAC-Bayes bounds comparisons}
\label{appendix:test_bounds}
A test bound is a confidence interval for the mean of the risk constructed from a sample of risk values $r_1, \cdots, r_n$. We assume that the risk values are evaluated on $n$-i.i.d. new datapoints, independent from the data used to train the predictor.
Test bounds can be inferred from some concentration inequalities in the following way. 
\begin{lemma}
\label{lem:appendix_ch2:conc_to_test_bounds}
Let $\mathcal{P}$ denote a subset of probability measures on $\mathbb{R}_+$. Consider a probability $\mathbb{P}\in\mathcal{P}$, denote $\overline{r}$ its mean. Consider $n$ -i.i.d. draws $r_1, \dots, r_n$ from $\mathbb{P}$ and denote $\overline{r}_n$ the empirical average $\overline{r}_n:= \frac{1}{n}\sum_{i=1}^n r_i$.
Consider a concentration inequality valid for all $\mathbb{P}\in\mathcal{P}^{\otimes n}$ stating that
$$\mathbb{P}\pth{\overline{r}_n \leq \overline{r} - t} \leq \gamma(t, \overline{r}, n, \mathcal{P}).$$

Define $\hat{r}_+(\delta, n) = C_\gamma(\overline{r}_n, n, \delta, \mathcal{P}) := \sup\{m, \gamma(m- \overline{r}_n, m, n, \mathcal{P})\geq \delta\}$. Then $[0, \hat{r}_+(\delta, n)]$ is a confidence interval of level at least $1-\delta$.
\end{lemma}
\begin{proof}
Consider the function
$$F(p, n,\delta, \mathcal{P}) = \max C(p, n, \delta) := \{\overline{r} \mid \exists \mathbb{P}\in \mathcal{P} \text{ s.t. } \mathbb{P}[\Id] = \overline{r}, ~~ \mathbb{P}^{n}\pth{\overline{r} \leq p} \geq \delta \}.$$
By definition of $F$, $[0, F(\overline{r}_n, n, \delta, \mathcal{P})$ is a confidence interval of level $1- \delta$. Hence if  $F(p,n, \delta, \mathcal{P}) \leq C_\gamma(p, n,\delta, \mathcal{P})$, then $[0, C_{\gamma}(\overline{X},n,\alpha)]$ is a confidence interval of level at least $1- \delta$.

To prove the inequality, consider a sequence $(m_i)_{i\in\mathbb{N}}$ such that
\begin{itemize}
\item $ p \leq m_i$,
\item $\forall i$, $m_i \in C(p,n, \delta)$,
\item $m_i\rightarrow F(p, n, \delta, \mathcal{P})$.
\end{itemize}

Since $F(p,n , \delta, \mathcal{P}) \geq p$, such a sequence always exists. Then for all $i$, there exists $\mu_i \in \mathcal{P}$ with mean $m_i$ such that $\mathbb{P}^n\pth{\overline{r}_n \leq m_i + p - m_i }\geq \delta$. Moreover, using the concentration inequality with $t = p-m_i$, it follows that $\gamma(m_i - p, m_i, n, \mathcal{P}) \geq \mathbb{P}^n\pth{\overline{r}_n \leq m_i + p - m_i }$. Hence $\gamma(m_i - p, m_i, n, \mathcal{P}) \geq \delta$ for all $i$. By definition of $C_\gamma(r, b, \delta, \mathcal{P})$, this implies the inequality.
\end{proof}
The quantity $F(p,n, \delta, \mathcal{P})$ introduced in the bound is the tightest upper bound on the average which can be inferred using the empirical average; the proof consists of showing that any concentration inequality results in a looser bound. The best concentration inequality of the form studied here should achieve this bound.

We will now consider the setting of bounded risks ($\mathcal{P}$ are probability measures on $[0,1]$), and use an extended version of Hoeffding's lemma to obtain a test bound. This extended version was in fact also established in Hoeffding's original paper \citep{Hoeffding1963}, and improves on the bound by considering the variance. This has a large impact when the empirical mean is small, since bounded random variables with average close to the bounds must have small variance. Notably, this results in test bounds with rate $1/n$ when the empirical mean is $0$ (compared to $1/\sqrt{n}$ for the classic Hoeffding's bound).

\begin{lemma}[Lemma 1 in \protect\citealp{Hoeffding1963}]
\label{app2:better_hoeffding}
Let $X$ be a random variable distributed from $\mathbb{P}$ taking value between $a$ and $b> a$. Then $\forall \lambda$,

\begin{equation}
\mathbb{E}\bk{\exp(\lambda X)} \leq \Psi_\lambda\pth{a + (b-a)\mathcal{B}\pth{\frac{\mathbb{E}\pth{X}-a}{b-a}}}
\end{equation}
where $\Psi_\lambda$ is the moment generating function evaluated at $\lambda$, i.e.
\begin{equation}
\Psi_\lambda\pth{a + (b-a)\mathcal{B}\pth{\frac{\mathbb{E}\pth{X}-a}{b-a}}} = \frac{\mathbb{E}[X]-a}{b-a}\exp(\lambda a) + \frac{b - \mathbb{E}[X]}{b-a}\exp(\lambda b) 
\end{equation}
\end{lemma}
Using the Chernov bound strategy used in the original Hoeffding's theorem and replacing Hoeffding's lemma with this improved version yields the following concentration inequality.
\begin{corollary}
\label{corol:appendix_ch2:varhoeffdingconcentration}
For $\mathbb{P}$ a measure on $[0,1]$ of mean $\overline{r}$, noting $\overline{r}_n$ the empirical mean for $n$ i.i.d. observations, it follows that
\begin{align*}
\mathbb{P}^n\pth{\overline{r}_n\leq  \overline{r}-t}&\leq \inf_{\lambda} \exp\pth{-n \lambda t + n \Psi_\lambda\pth{\mathcal{B}\pth{\overline{r}}}}\\
&=\exp\pth{-n \pth{(\overline{r} - t)\log\pth{\frac{\overline{r} - t}{\overline{r}}} + (1- \overline{r} + t)\log\pth{\frac{1- \overline{r} + t}{1- \overline{r}}}}}.
\end{align*}
\end{corollary}

The test bounds proposed in table \Cref{table:test_bounds_vs_PACBayes} are computed using the test bounds derived from the concentration inequality of \Cref{corol:appendix_ch2:varhoeffdingconcentration} using the concentration to test bound algorithm described in \Cref{lem:appendix_ch2:conc_to_test_bounds}. The inversion is performed using dichotomy. We remark that the slightly tighter test bounds could be obtained in the classification setting by limiting the space $\tilde{P}$ to Binomial distributions and explicitly computing the tightest confidence interval achievable in that setting.

\begin{table}[htbp]
\small
\caption{Comparison of the generalisation guarantees obtained using PAC-Bayes with a data dependent prior and test bounds on the prior mean. Lines Spambase, Bioresponse, Har, Electricity, Mammography and MNIST 1 come from \cite{perezortiz2021learning} (table 2, without validation). Lines MNIST FCN and MNIST MCN corresponds to FCN and CNN architecture respectively, trained using criteria $f_{\text{quad}}$ from table 1 in \cite{perez-ortiz2021a}. Lines CIFAR-nL-p come from table 5 in the same source, with nL indicating the number of layers and p the fraction of data used to train the prior. The objective function considered is the one resulting in the tightest PAC-Bayes bound. In every case, the loss considered is the classification error. The PAC-Bayes bound and test bound are computed for a confidence level of $0.965$.}
\vspace{1em}
\centering
\begin{tabular}{|l|c|c|c|c|}
\hline
Database & PAC-Bayes bound & Test bound & Test score & $n_{\text{valid}}$\\
\hline
Spambase     &0.140  &\textbf{0.0941}&0.077 &1840\\
Bioresponse  &0.318  &\textbf{0.291}&0.261 &1500\\
Har          &0.035  &\textbf{0.0307}&0.024 &4119\\
Electricity  &\textbf{0.223}  &0.2290&0.221 &18124\\
Mammography  &0.022  &\textbf{0.0202}&0.015 &4473\\
MNIST 1      &0.034  &\textbf{0.0274}&0.025 &30000\\
MNIST FCN    &0.0279 &\textbf{0.0224}&0.0202&30000 \\
MNIST CNN    &0.0155 &\textbf{0.0120}&0.0104&30000\\
CIFAR-9L-50  &0.2901 &\textbf{0.2583}&0.2518&30000\\
CIFAR-9L-70  &0.2377 &\textbf{0.2249}&0.2169&18000\\
CIFAR-13L-50 &0.2127 &\textbf{0.1973}&0.1914&30000\\
CIFAR-13L-70 &0.1758 &\textbf{0.1649}&0.1578&18000\\
CIFAR-15L-50 &0.1954 &\textbf{0.1744}&0.1688&30000\\
CIFAR-15L-70 &0.1667 &\textbf{0.1560}&0.1490&18000\\
\hline
\end{tabular}
\label{table:test_bounds_vs_PACBayes}
\end{table}
\clearpage
\FloatBarrier
\section{Numerical evaluations of the prior requirements}
\label{app:graphs_catoni}
We evaluate the prior requirements of Catoni's bound obtained in \Cref{sec:Cat_bound_prior_requirement} for a MNIST like setting. The number of samples $n$ is set to 60 000, and the confidence level is chosen, as in \cite{perezortiz2021learning} to 0.035. The target generalisation level is set to $0.015$ (1.5\% misclassification), similar to the PAC-Bayes generalisation guarantee reported in this \cite{perez-ortiz2021a}. In \Cref{figure:Q_opt_vs_Q_asympt_opt}, we compare the temperature free quantile requirement $\overline{Q}_{\textup{Cat}}$ to the quantile requirement at the temperature obtaining the lowest asymptotic quantile requirement, $\overline{Q}_{\textup{Cat}, \lambda_{\textup{opt}}}$. In \Cref{figure:Q_opt_vs_Q_asympt_opt}, we compare $\overline{Q}_{\textup{Cat}}$ to the quantile requirement for temperatures varying between $\lambda_{\textup{min}}$ and $\lambda_{\textup{max}}$.

\begin{figure}[htpb]
\centering
\includegraphics[width=\linewidth]{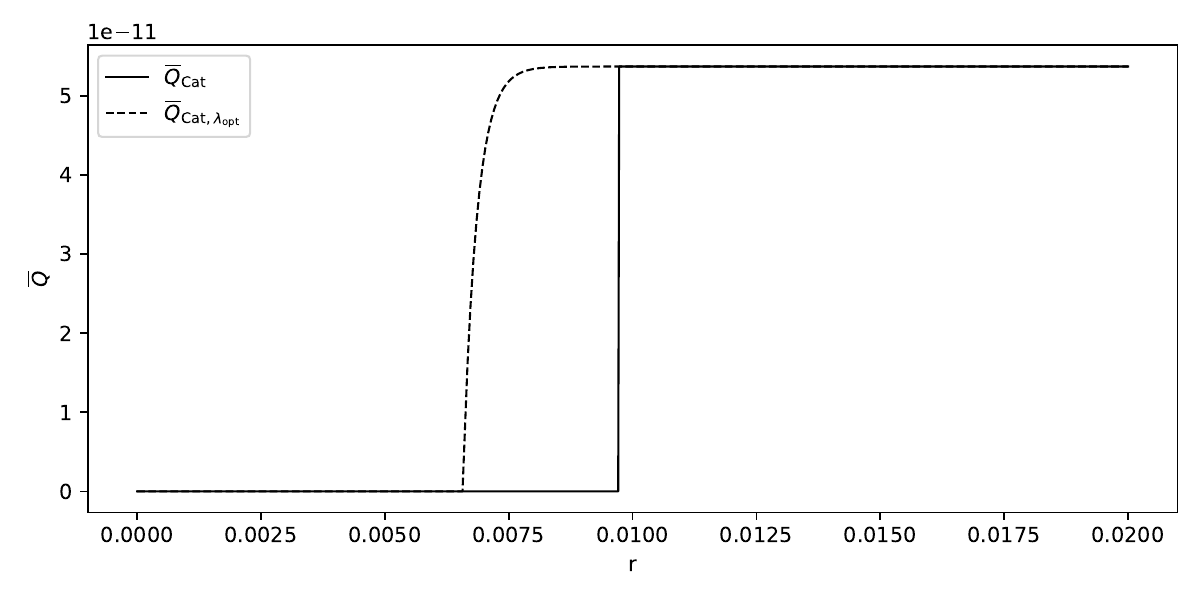}
\caption{Comparison of $\overline{Q}_{\textup{Cat}, \lambda_{\textup{opt}}}$ and $\overline{Q}_{\textup{Cat}}$. The target generalisation gap is fixed to $G=0.015$, the number of observations to 60 000 and the confidence level to $1 - \delta = 1- 0.035$. The temperature free requirement $\overline{Q}_{\textup{Cat}}$ exhibits a phase transition at \ensuremath{r_{\textup{thresh}} = G - 2\sqrt{-\log(\delta)/\pth{8n}} \simeq 0.009714}. While $\overline{Q}_{\textup{Cat},\lambda_{\textup{opt}}}$ provides a good approximation of the point wise minimum for large value of $r$, it leads to more restrictive quantile requirements for $r<r_{\textup{thresh}}$. This graph implies that no prior putting less than 5.3e-11 mass on predictors with risk smaller than 0.01 can hope to obtain a generalisation guarantee higher than 0.015 valid with confidence level of 0.965 by training Catoni's bound on datasets with 60 000 samples (such as MNIST).}
\label{figure:Q_opt_vs_Q_asympt_opt}
\end{figure}

\begin{figure}[htbp]
\centering
\includegraphics[width=\linewidth]{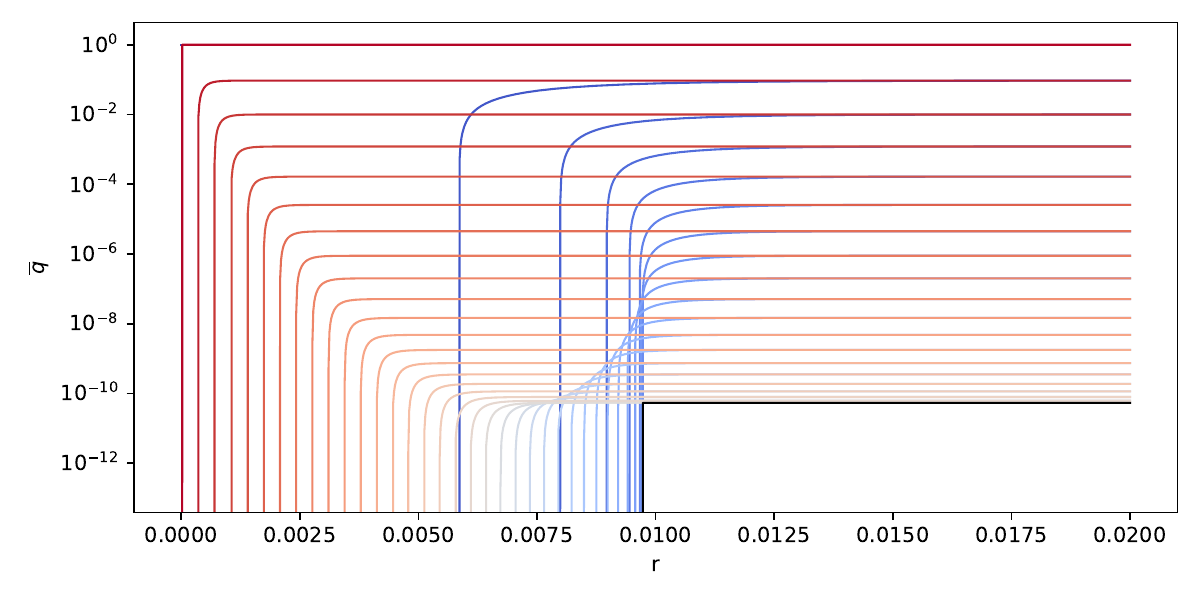}
\caption{Evaluation of \ensuremath{\overline{Q}_{\textup{cat},\lambda}} as a function of $r$ for different temperatures. The target generalisation gap is fixed to \ensuremath{G=0.015}, the number of observations to 60 000 and the confidence level to \ensuremath{1 - \delta = 1- 0.035}. 40 temperatures between \ensuremath{\lambda_{\textup{min}}} and \ensuremath{\lambda_{\textup{max}}} are assessed (blue denotes lower temperature, red larger temperature). In black, the minima of the risk requirement over all temperatures is plotted.}
\label{figure:Q_catoni}
\end{figure}

\clearpage
\FloatBarrier
\section{Strength of the quantile requirements}
\label{app:strength_quant_req}
The prior quantile conditions resulting from \cref{corol:prior_condition_protocol} are typically hard to interpret as they involve extremely low quantities (\emph{e.g.} 1e-10). As such, it can be difficult to decide whether such conditions are really restrictive, or virtually always met. We provide here some insights on the probability for an uninformed prior to draw low risk predictors, based on a simplified classification setting.

We consider a classification setting with $k$ classes. We assume that each of these classes are decomposed into $p$ sub-classes, called clusters. The assumption that the prior is uninformed can be interpreted as an assumption that the prior distribution is invariant to class permutation. For simplicity's sake, we here assume that the prior distribution draws for each cluster a class at uniformly at random. We make the further assumption that each cluster contains the same number of data.

The above sets of assumption form an optimistic setting, as we assume that the prior considers only perfect classification partitions at the cluster level. On the other hand, it also considers that clusters belonging to the same class are no more similar than clusters belonging to different class.

In this setting, the distribution of the error follows a Binomial distribution with probability of success (\emph{i.e.} misclassification) $\frac{k-1}{k}$ and $k \times p$ draws. The probability of observing an error less than $r <\frac{1}{2}$ can be either exactly computed or upper bounded by 
$$\prior\pth{\risk \leq r}\leq \exp\left(\log(rk p) + r k  p\log(k p) - (1-r) k p \log(k)\right).$$

The probability that the uninformed prior obtains less than $r= 0.015$ error (\emph{i.e.} less than 1.5\% of misclassification) is, in the 10 classes setting with 2 cluster per class, less than 1.47 e-20. This decreases to less than 4.7 e-96 when considering 10 cluster per class, and 2.45 e-378 when considering 40 cluster per class. The probability of obtaining an error smaller than $0.015$ for various number of classes and number of cluster per class is represented in \Cref{figure:quantile_uninformed_prior}. These indicate that even such low prior mass requirements of order less than $1e-300$ might be hard to meet for uninformed priors, especially as the number of classes increases.

\begin{figure}[htbp]
\centering
\includegraphics[width=\linewidth]{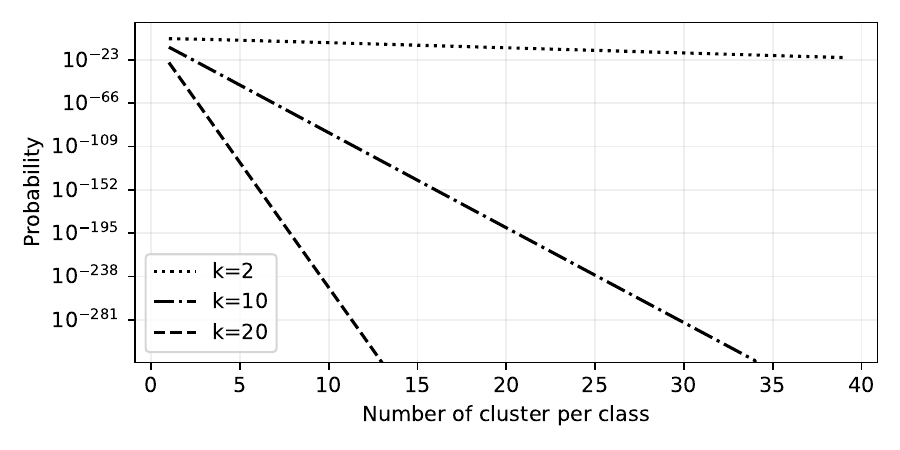}
\caption{Evaluation of \ensuremath{\prior\pth{\risk \leq 0.015}} in the uninformed prior setting for varying number of class and number of clusters.}
\label{figure:quantile_uninformed_prior}
\end{figure}